\documentclass[preprint,12pt]{elsarticle} 
\usepackage{a4wide}
\usepackage{setspace}
\usepackage{xspace}
\usepackage{algorithm,algpseudocode}
\usepackage[T1]{fontenc}
\usepackage[utf8]{inputenc}

\usepackage[T1]{fontenc}
\usepackage{multirow}
\usepackage[utf8]{luainputenc}
\usepackage{graphicx}
\usepackage{subfig}
\usepackage{captdef}
\usepackage{float}
\usepackage{hyperref}
\usepackage{amsmath, amsthm, amssymb, bbm}

\newtheorem{thm}{Theorem}[section]

\newtheorem{rem}{Remark}[section]

\newtheorem{lem}{Lemma} [section]

\newtheorem{hip}{Hypothesis}[section]

\newcommand{\R}{\mathbb{R}}

\newcommand{\MATLAB}{\textsc{Matlab}\xspace}

\begin{document}
\begin{frontmatter}
\title{Diffusion Representation for Asymmetric Kernels}

\author[1]{Alvaro Almeida Gomez} 
\emailauthor{aag182@impa.br}{A. A. Gomez}
\author[2]{Antônio J.  Silva Neto}
\emailauthor{ajsneto@iprj.uerj.br}{A. J.  Silva Neto}
\author[1,3]{Jorge P. Zubelli \corref{mycorrespondingauthor}}
\emailauthor{jorge.zubelli@ku.ac.ae}{J. P. Zubelli}
\address[1]{IMPA, Est. D. Castorina, 110, Jardim Bot\^anico, Rio de Janeiro, 22460-320, Brazil}
\address[2]{IPRJ-UERJ, R. Bonfim 25, Nova Friburgo 28625-570, Brazil}

\address[3]{Khalifa University, P.O. Box 127788, Abu Dhabi, UAE}

\cortext[mycorrespondingauthor]{Corresponding author}

\makeatletter
\numberwithin{equation}{section}
\numberwithin{figure}{section}

\begin{abstract}
We extend the diffusion-map formalism to data sets that are induced by asymmetric kernels. Analytical convergence results of the resulting expansion are proved, and an 
algorithm is proposed to perform the dimensional reduction. 
%
%

A coordinate system connected to the tensor product 
of Fourier basis is used to represent the underlying geometric structure
obtained by the diffusion-map, thus reducing the dimensionality of the data set and making use of the speedup provided
by the two-dimensional Fast Fourier Transform algorithm (2-D FFT). 

We compare our results with those obtained by other eigenvalue expansions, and verify the efficiency of the algorithms with synthetic data, as well as with real data from applications including climate change studies.
\end{abstract}
\begin{keyword}{Diffusion maps; Dimensional reduction; FFT; Asymmetric kernels}
\end{keyword}
\end{frontmatter}


\section*{Highlights}

\begin{itemize}

\item  A novel methodology for asymmetric-kernel  data  dimension reduction is developed.

\item Dimension reduction is based on the highly efficient FFT algorithm in higher dimensions.

\item Numerical evidence indicates that the tensor product of the FFT basis performance is faster than eigenvalue-based methods.

\item Geometric features of complex data sets are revealed in key examples such as the M\"obius strip and the sphere.

\item 
The methodology is employed in meteorological applications to identify regions of largest temperature variation.

\end{itemize}

\section{Introduction}

 Data compression has been studied extensively in many applications. See \cite{SLEIJPEN20101100,WU2021101}. Several dimensionality reduction algorithms are based on the spectral decomposition of symmetric linear operators which induce geometric structures in the data set. 
 A classical example of such operators is the Laplacian matrix associated to an undirected graph \cite{DAS2004715}. In  Ref. \cite{6789755}, the eigenvalue decomposition of the Laplacian matrix is used to reduce the dimensionality of data sets in such way  that the local information is preserved. The diffusion-map  approach is based on using the symmetric
 normalized Laplacian matrix  \cite{COIFMAN20065}. Compared with other dimensionality reduction methods diffusion-map assumes that the data set resides in a lower dimensional manifold, and uses approximations of the Laplace-Beltrami operator to reveal relevant parameters in the data set.

\par The spectral decomposition theorem does not hold for integral operators with asymmetric kernels. Therefore, we cannot use the diffusion-map framework to represent diffusion distances induced by asymmetric kernels. Moreover, if we use the spectral decomposition in some symmetric normalization of an asymmetric kernel, its performance requires a computational complexity $O(n^3)$, for $n \times n$ matrices.
 \par In order to reduce the above $O(n^3)$ complexity and to deal with more general  kernels, we present a new framework to represent the diffusion geometry induced by asymmetric kernels. We use the 2-D FFT to compute this  representation. The main advantage of using this representation is that compared to the eigenvector representation, the computation time decreases. In fact, for a matrix of dimension $n \times n$,  the complexity of the 2-D FFT is $O(n^2 \times log(n))$.
 The choice of  the 2-D FFT to represent the data set structure is based on the fact that the  Fourier basis diagonalizes the Laplacian  defined on the Euclidean Torus.

 \par In this paper, we deal with data sets whose structure is induced by an asymmetric kernel. Based on Refs.~\cite{HAJJI20041011, SALHOV2018324, SALHOV2015399},  we use alternative orthonormal bases to reduce the dimensionality of the data sets in such way that the geometric structure is preserved. 
 Here, we work with the representation theory of diffusion distances in the context of changing data, proposed in Ref. \cite{COIFMAN201479}. First, we find a representation form for the classical diffusion geometry, and then we  find a representation form for changing data. To do that, we start with the case $t=1$ and then extend it for any time $t$. 
\par This paper is organized as follows, in Section  \ref{histor} we give a brief exposition of the classical representation theory for diffusion distances proposed in Refs. \cite{COIFMAN20065,COIFMAN201479,Coifman7426,MARSHALL2018709}. In Section  \ref{resultados} we present our framework to represent  diffusion distances when the structure in the data set is induced by an asymmetric kernel. Finally, in Sections \ref{experimentos} and  \ref{conclu}, we show some experiments with applications and results, as well as draw some conclusions and directions for further research.

\section{Classical Diffusion-Map Theory} \label{histor}
In this section, we review some classical results on diffusion-map theory. We define  diffusion distances in a measure space, and we recall some results related to the representation of such diffusion distances.
\subsection{Diffusion distance}
Assume that our data set $(X,\mu)$ is a measure space, and let $k:X \times X \to \R_{\ge 0}$ be  a non-negative symmetric kernel, which is used to measure the local connectivity between two points $x$ and $y$.  If $X$ is also a metric space, the most classical example of these kernels is the Gaussian kernel given by
$$e^{-d^2(x,y) / 2 \sigma^2},$$
where $d$ is the distance function, and $\sigma^2$ is the scaling parameter.
We define the associated Markov kernel  $\rho(x,y)$ by
\begin{equation}
  \label{kerma}  \rho(x,y):=\frac{k(x,y)}{\sqrt{v(x)} \sqrt{v(y)}}
\mbox{ ,} \end{equation}
where $v(x)$ is the volume form defined as $v(x)=\int_{X} k(x,y) \, d\mu (y)$. Assuming that the volume form never vanishes and that $k \in  L^2(X \times X) $, then the operator $A:\, L^2(X) \to \, L^2(X)$ given by
\begin{equation}
  \label{opA}  A(f)(x):=\int_{X} \rho(x,y)\,f(y) \,d \mu (y)  , 
\end{equation}
is compact and self-adjoint. According to the spectral theorem, we can write
$$\rho(x,y)=\sum_{n \in \mathbb{N}} \lambda_{n} \phi_{n}(x) \ \phi_{n}(y)\mbox{ , } $$
where $ \{  \phi_{n}, \lambda_{n}\}$ is the spectral decomposition of the operator $A$. For any natural number $t$, we define the diffusion distance at time $t$ between two points $x$ and $y$ by
\begin{equation}
\label{didior} D^t (x,y):= {\| \, \rho ^ t(x,\cdot)-\rho ^ t(y,\cdot)\|}_{L^2(X)}
\mbox{ ,} \end{equation}
where $\rho^ t$ is the kernel of the integral operator $A^ t$. Here, $A^ t$ is the composition of the operator $A$, a total of $t$ times. The kernel $\rho^ t(x,y)$ measures the probability that the points $x$ and $y$ are connected by a path in time $t$. Observe that the distance $D^t$ is an average over all of the paths in time $t$ connecting $x$ to $y$. Therefore, the diffusion distance is robust to noisy data. We see that the quantity $D^t$ is small when there are many paths of length $t$ connecting $x$ and $y$.
Using the spectral decomposition of the operator $A$, we can write the diffusion distance as
\begin{equation}
 \label{spede} D^t (x,y)=\sqrt{\sum_{n \in \mathbb{N}} \lambda_{n}^t (\phi_{n}(x)-\phi_{n}(y))^2}
\mbox{ .} \end{equation}
\par The above expression allows us to reduce the dimensionality of the diffusion geometry, namely, we embed our data set in a lower dimensional space using the diffusion-map $\psi_{k}^t : X \, \to \mathbb{R}^k $, where 
$$\psi_{k}^t(x)=( \lambda_{i}^{t/2} \, \phi_{i}(x))_{i=1}^k \mbox{ .}$$

\subsection{Diffusion-map for changing data}
The diffusion distance for changing data proposed in Ref. \cite{COIFMAN201479} compares data points between parametric data sets. We define $X_\alpha$ as the data set $X$ endowed with the kernel structure $k_\alpha$. As above, for each kernel $k_\alpha$, we consider the associated Markov kernel $\rho_\alpha$ defined in Eq.~(\ref{kerma}), and the operator $A_\alpha$ as in Eq.~(\ref{opA}). To compare the data structure $X_\alpha$ with $X_\beta$, we define the dynamic diffusion distance  $D^t:  X_\alpha \times  X_\beta \to \R_{\ge 0} $  by
$$ D^t(x_\alpha,y_\beta)={\| \, \rho_\alpha ^ t(x,\cdot)-\rho_\beta ^ t(y,\cdot)\|}_{L^2(X)} \mbox{ . } $$
Furthermore, in Ref. \cite{COIFMAN201479}, the global diffusion distance  at time $t$ is defined by
$$\mathbb{D}^t(X_\alpha,X_\beta)=\sqrt{\int_{X}(D^t(x_\alpha,x_\beta))^2 \, d\mu(x)} \mbox{ . } $$
The global diffusion distance measures the change from the data structure $X_\alpha$ to  $X_\beta$.
Under mild assumptions on the data set $X$ and the family of kernels $\{k_\alpha\}$, the dynamic diffusion distance and the global diffusion distance can be computed using the spectral decomposition of the operators $A_\alpha$, as in Eq.~(\ref{spede}). See Ref. \cite{COIFMAN201479}.

\section{Diffusion Representation for Asymmetric Kernels} \label{resultados}

Assume that $(X, \mu)$ is a measure space, and consider an asymmetric kernel $k$, which is any square integrable measurable function  $k:X \times X \to \R_{\ge 0}$. Observe that here we do not require that the kernel $k$ be symmetric. As example of these kernels, we assume that $(X,d)$ is a metric space, and consider the weighted Gaussian kernel defined by
$$w(x,y) \; e^{-d^2(x,y)/2 \sigma^2}\mbox{ , } $$
where $w$ is the weight function. Weighted Gaussian kernels measure how information is distributed locally from $x$ to $y$.
Note that the distribution of the information may not be uniform.
In order to deal with more general models, we do not use the Markov normalization given by Eq.~(\ref{kerma}) to define the diffusion distance, instead, we use the diffusion kernel $k$ to define it. More specifically, we work with the diffusion distance at time $t$ given by 
\begin{equation} \label{difudistnu}
  D^t (x,y):= {\| \, k^ t(x,\cdot)-k ^ t(y,\cdot)\|}_{L^2(X)}
\mbox{ ,} \end{equation}
where $k^ t$ is the kernel of the operator $A^ t,$ and $A$ is the integral operator defined as
\begin{equation}
  \label{opeA}  A(f)(x):=\int_{X} k(x,y)\,f(y) \,d \mu (y) .   
\end{equation}

\subsection{Diffusion representation for t=1} 
We now design a representation for the diffusion distance given by Eq.~(\ref{difudistnu}), where  $k$ is an asymmetric kernel. Suppose that $\{W_{m_1}\}_{m_1  \in  \mathbb{Z}}$ and $\{W_{m_2}\}_{m_2  \in  \mathbb{Z}}$ are two orthonormal bases of $L^2 (X),$ and that there exists a positive constant $M$ such that for any $m_1  \in  \mathbb{Z},$ and all $x \in X$, 
$$ \|W_{m_1}(x)\| \le M.$$
\par We recall that in such case the tensor product $\{W_{m_1} \otimes W_{m_2}  \}_ {(m_1,m_2)  \in  \mathbb{Z} \times \mathbb{Z}}$ defined by
\begin{equation}
    \label{tensorbasis}W_{m_1} \otimes W_{m_2} (x,y)= W_{m_1}(x) \,{W}_{m_2}(y),
\end{equation}
is an orthonormal basis of  $L^2 (X \times X)$, (for more details, see \cite{functio,grafakos2014classical,steinsingu}).
We note, in passing, that our approach also works if the $L^2 (X)$ space is finite dimensional, in which case $\mathbb{Z}$ should be substituted by a finite index set.
To develop our theory we assume the following hypothesis on the kernel $k$.
\begin{hip}
\label{hipo}
Suppose that $k \in L^2 (X \times X),$ and for a.e $x\in X,$ the kernel function $k(x, \cdot)$ belongs to the space $L^2 (X)$.
\end{hip} 
If we assume the above hypothesis, we can use the  basis given by Eq.~(\ref{tensorbasis}) to write the kernel $k(x,y)$ as
\begin{equation}
\label{fourep}
    k(x,y)= \sum_{(m_1,m_2)  \in  \mathbb{Z} \times \mathbb{Z}} a_{(m_1 , m_2 )} \, W_{m_1} \otimes W_{m_2} (x,y),
\end{equation}
where $a_{(m_1 , m_2 )}$ are the coefficients. Using this decomposition we obtain a representation form for the diffusion distance at time $t=1$.

\begin{thm}[Diffusion representation for $t=1$] \label{difurepteo}
Assume that the kernel $k$ satisfies the Hypothesis  \ref{hipo}, and that the  representation of $k$ in the coordinate system (\ref{tensorbasis}) is given by Eq.~(\ref{fourep}). If the coefficients satisfy the summability condition
\begin{equation}
    \label{condilun}\sum_{(m_1,m_2)  \in  \mathbb{Z} \times \mathbb{Z}} \| a(m_1,m_2) \| < \infty,
\end{equation}
then the diffusion distance  at time $t=1$ has the representation form
\begin{equation} \label{reprefor}
(D^1 (x,y))^2=\sum_{m_2 \in  \mathbb{Z}} \| \sum_{m_1 \in  \mathbb{Z}}  a(m_1,m_2) \, ( W_{m_1}(x)-W_{m_1}(y))\| ^2
\mbox{ .} \end{equation}
\end{thm}
\begin{proof}
Using the summability condition of Eq.~(\ref{condilun}), we can write the function $k(x,\cdot)-k(y,\cdot) $ as
$$ \sum_{m_2 \in  \mathbb{Z}}\left(\sum_{m_1 \in  \mathbb{Z}} a(m_1,m_2) \, ( W_{m_1}(x)-W_{m_1}(y)) \right) W_{m_2}(\cdot).$$
Since the set $\{W_{m_2}\}$ is an orthonormal basis, we conclude that 
$$(D^1 (x,y))^2=\sum_{m_2 \in  \mathbb{Z}} \| \sum_{m_1 \in  \mathbb{Z}}   a(m_1,m_2) \, ( W_{m_1}(x)-W_{m_1}(y))\| ^2.$$
\end{proof}
For practical purposes, we do not use the representation formula given in Theorem~\ref{difurepteo} to approximate the diffusion distance, because this representation includes two sums with many terms. Instead, we use an approximation involving sums of few terms, this is established in Theorem~\ref{teoprin}. In order to prove this theorem, we first prove an auxiliary lemma.
\begin{lem} [Approximation lemma]  \label{aproximalem}  Consider the function  $$f_{k_1,k_2}:L^{1}(\mathbb{Z} \times \mathbb{Z} ) \times L^{1}(\mathbb{Z} \times \mathbb{Z} ) \times X \times X \to \R_{\ge 0},$$
defined by 
$$ f_{k_1,k_2}(a_1,a_2,x,y):= \sum_{\|m_2\|\le k_2} \| \sum_{\|m_1\|\le k_1}  a_1(m_1,m_2) \, W_{m_1} (x)-a_2(m_1,m_2) \, W_{m_1}(y) \| ^2 \mbox{ .}$$
Suppose that $a_1$ and $a_2$ are  sequences in $L^{1}(\mathbb{Z} \times \mathbb{Z})$, then for each $\delta >0,$ there exist  positive integers $\overline{k_1}$ and $\overline{k_2}$, such that if  $\overline{k_1} \le k_1,$  $\overline{k_2} \le k_2,$  and if  $x,y \in X,$ the following inequality holds
$$\|f_{k_1,k_2}(a_1,a_2,x,y)- f(a_1,a_2,x,y)\| < \delta\mbox{ , }$$
where $f(a_1,a_2,x,y)=\lim_{k\to\infty} f_{k,k}(a_1,a_2,x,y).$

\end{lem}

\begin{proof}
Let $b(m_1,m_2)$ be defined by  $$b(m_1,m_2)=a_1(m_1,m_2) \,  W_{m_1}(x) -a_2(m_1,m_2) W_{m_1}(y).$$ 
 Since the  $L^{2}(\mathbb{Z} \times \mathbb{Z} )$  norm is smaller than, or equal to, the $L^{1}(\mathbb{Z} \times \mathbb{Z} )$  norm, we obtain that the expression $\|f_{k_1,k_2}(a_1,a_2,x,y)- f(a_1,a_2,x,y)\|$ is bounded from above by
$$ \left(  \sum_{\|m_2\|\le k_2} \sum_{\|m_1\|> k_1}\| b(m_1,m_2) \|  + \sum_{\|m_2\|> k_{2} }  \sum_{m_1  \in  \mathbb{Z}} \|b(m_1,m_2) \| \right)^2  . $$
The above expression is dominated by
$$  \left( M \sum_{(m_1,m_2) \notin B(0,k_1) \times B(0,k_2) }  \| a_1(m_1,m_2)\|+\| a_2(m_1,m_2)\| \right)^2.$$
Therefore, by Assumption (\ref{condilun}), we conclude that for a given $\delta > 0,$ we can take $k_1$ and $k_2$ large, such that the above term is less than or equal to $\delta.$ 
\end{proof}
As a consequence of the above lemma, we prove the following theorem, which states that we can approximate the diffusion distance with finite sums.

\begin{thm} \label{teoprin}
Assume that the kernel $k$ satisfies the same hypotheses of Theorem~\ref{difurepteo}. Then, for each $\delta >0,$ there exist  positive integers $\overline{k_1}$ and $\overline{k_2}$ such that for any $\overline{k_1} \le k_1,$ and $\overline{k_2} \le k_2,$ and all $x,y \in X$, the following inequality holds
$$ |f_{k_1,k_2}(a_1,a_1,x,y)-(D^1 (x,y))^2| \le \delta \mbox{ .}$$
\end{thm}

\begin{proof}
The proof is a consequence of Theorem~\ref{difurepteo} together with Lemma~\ref{aproximalem}.
\end{proof}

\subsection{Diffusion representation for arbitrary time}
Suppose that $t$ is a positive integer denoting an arbitrary time. We now use the coordinate system of Eq.~(\ref{tensorbasis}) to find the representation form for the kernel $k^{t+1}$  in terms of the coefficients $a(m_1,m_2)$. Let $k$ be an asymmetric kernel, and suppose that for any $1\le j \le t,$ the kernel $k^j$ satisfies all the hypotheses of Theorem~\ref{aproximalem}. Under this assumption, we can use the Fubini's theorem to write recursively the kernel $k^{t+1}$  of the operator $A^{t+1}$ as
\begin{equation}
    \label{recufor} k^{t+1}(x,y)=<k(x, \cdot),k^{t}( \cdot , y )>_{L^2(X)} 
\mbox{ .} \end{equation}
\par Assuming that the kernel $k^{t}$ has the series representation
$$k^{t}(x,y)= \sum_{(m_1,m_2)  \in  \mathbb{Z} \times \mathbb{Z}} a_{(m_1 , m_2 )}^{t} \,W_{m_1} \otimes W_{m_2} (x,y) \mbox{. }$$
then by Eq.~(\ref{recufor}) we obtain that 
\begin{equation}
    \label{recuco} a_{( m_1 , m_2 )}^{t+1}=\sum_{k  \in  \mathbb{Z}}  a_{( m_1, k )} \, a_{( k , m_2)}^{t}
\mbox{ .} \end{equation}
Recursively, we obtain that the expression for the coefficients of the kernel $k^{t+1}$ 
\begin{equation} 
\label{expansion}
    a_{( m_1 , m_2 )}^{t+1}= \sum_{n_1  \in  \mathbb{Z}} \sum_{n_2  \in  \mathbb{Z}} \sum_{n_3  \in  \mathbb{Z}}\dots  \sum_{n_t  \in  \mathbb{Z}} a_{( m_1 , n_1)} \, a_{( n_1 ,n_2)}   a_{( n_2 ,n_3)}\dots a_{( n_t ,m_2)}
\mbox{ .} \end{equation}
Again, the above expression  contains infinitely many sums. We now prove that  we can approximate the coefficients of $k^{t}$ using finite sums. The following lemma establishes this result.

\begin{lem} \label{poqueal}
For any $\delta > 0$, there exist $k_0$, such that for any $k_0\le k$   we have 
$$\|a_{}^{t+1}-h^{t+1}_k(a) \|_{L^1(\mathbb{Z} \times \mathbb{Z})} \le \delta \mbox{ ,}$$
where $h^{t+1}_k(a) (m_1,m_2)$ is the finite sum
$$ h^{t+1}_k(a) (m_1,m_2)= \sum_{\|n_1\|\le k} \sum_{\|n_2\|\le k} \sum_{\|n_3\|\le k}\dots  \sum_{\|n_t\|\le k} a_{( m_1 , n_1)} \, a_{( n_1 ,n_2)}   a_{( n_2 ,n_3)}\dots a_{( n_t ,m_2)}\mbox{ . } $$ 
\end{lem}

\begin{proof}
We prove by induction over $t$, for $t=1$ is clear since 
\begin{align*}
\sum_{(m_1,m_2)  \in  \mathbb{Z} \times \mathbb{Z} }\lvert \lvert a_{( m_1 , m_2 )}^{2}-h^{2}_k (m_1,m_2)\rvert \rvert & \le \sum_{\lvert \lvert n_1 \rvert \rvert \ge k} \sum_{m_1  \in  \mathbb{Z} } \sum_{m_2  \in  \mathbb{Z} }  \lvert \lvert a_{( m_1,n_1)} a_{(n_1,m_2)} \rvert \rvert  \\
 &\le  \lvert \lvert a \rvert \rvert_{L^1} \sum_{\lvert \lvert n_1 \rvert \rvert \ge k}  \sum_{m_1  \in  \mathbb{Z}  } \lvert \lvert a_{( m_1,n_1)} \rvert \rvert  \mbox{ .}
\end{align*}
Then, by Assumption  \ref{condilun}, we have that for $k$ large the above inequality is less than or equal to $\delta.$ We now assume that the claim holds for $t$, and we prove that also holds for $t+1$. For $k$ large we have that
$$\|a^{t}-h^{t}_k(a) \|_{L^1(\mathbb{Z} \times \mathbb{Z})} \le \delta.$$
The above inequality implies 
$$\|h^{t}_k(a) \|_{L^1(\mathbb{Z} \times \mathbb{Z})} \le (\delta + \| a \|_{L^1(\mathbb{Z} \times \mathbb{Z})} ) \mbox{ .}$$
Furthermore, by Eq.~(\ref{recuco}), we have that $$\|a_{( m_1 , m_2 )}^{t+1}-h^{t+1}_k (a) (m_1,m_2)\|,$$ is less or equal to
$$ \sum_{i\in I} \| a_{( m_1 ,i )}\| \, \|a_{( i ,m_2 )}^{t}-h^{t}_k (a) (i,m_2)\| + \sum_{\|i\|\ge k} \| a_{( m_1 ,i )}\| \|h^{t}_k (a) (i,m_2)\| \mbox{ .}$$
 Using the above inequalities, we obtain the following estimate 
$$\|a^{t+1}- h^{t+1}_k(a)\|_{L^1} \le \delta \|a\|_{L^1} + (\delta + \| a \|_{L^1}) \sum_{\lvert \lvert i \rvert \rvert \ge k}  \sum_{m_1  \in  \mathbb{Z}  } \lvert \lvert a_{( m_1,i)} \rvert \rvert\mbox{ . } $$
Therefore, by Assumption~(\ref{condilun}), we conclude that the claim holds for $t+1$. \end{proof}
We now use the above result to design a representation for the diffusion distance at time $t+1$. This representation is based on finite sums of the coefficients $a_{(m_1,m_2)}$. We establish this result in Theorem~\ref{teoreprecualtiempo}, the proof involves the following technical lemma.
\begin{lem}[Continuity] \label{continuitylemma}
Consider the function  $f_{k_1,k_2}$ as in Lemma~\ref{aproximalem}, and suppose that $a$ and $b$ are two sequences in $L^{1}(\mathbb{Z} \times \mathbb{Z} ).$ Then, for any positive number $\epsilon$, there exists a positive number $\delta$, such that for any pair of sequences $c$ and $d$ satisfying $$\|a-c\|_{L^{1}(\mathbb{Z} \times \mathbb{Z} )} \le \delta \hspace{1cm} \mathrm{and} \hspace{1cm}  \|b-d\|_{L^{1}(\mathbb{Z} \times \mathbb{Z} )} \le \delta \mbox{, }$$ then, the following inequality holds for all  $x,y \in X,$ and for all positive integers $k_1, k_2,$
$$ |f_{k_1,k_2}(a,b,x,y)-f_{k_1,k_2}(c,d,x,y) |\le \epsilon \mbox{. }$$
\end{lem}

\begin{proof} Let $k_1,k_2$ be positive integers, and define the function $R_{m_2}$ as
$$ R_{m_2}(a_1,a_2,x,y) := \sum_{\|m_1\|\le k_1}  a_1(m_1,m_2) \, W_{m_1} (x)-a_2(m_1,m_2) \, W_{m_1}(y).$$
Observe that
\begin{equation}
    \label{estimacontilem}\|R_{m_2}(a,b,x,y)-R_{m_2}(c,d,x,y)\|\le  M  \sum_{\|m_1\|\le k_1}  (\|(a-c)(m_1,m_2)\|+\|(b-d) (m_1,m_2)\|  ) 
\mbox{ ,} \end{equation} 
where $M$ is a constant independent of $x$ and $y$. For any real numbers $B$ and $C$, the following inequality holds
\begin{equation}
\label{desiguimport}|B^2-C^2|  = |(B-C)^2+2BC-2C^2|
 \le  (B-C)^2+2|C||B-C|
\mbox{ .} \end{equation}
Applying the above inequality to $B=\|R_{m_2}(c,d,x,y)\|,$ and $C=\|R_{m_2}(a,b,x,y)\|$, together with the fact that the $L^2$ norm is  smaller than or equal to the $L^1$ norm, we obtain that $|f_{k_1,k_2}(a,b,x,y)-f_{k_1,k_2}(c,d,x,y) |$ is bounded from above by
$$ M(E^2+2(\|a\|_{L^1}+\|b\|_{L^1})E) \mbox{, }$$
where $M$ is a constant which does not depend on $k_1, k_2, x, y,$ and 
$$E:=\|a-c\|_{L^{1}}+\|b-d\|_{L^{1}} \mbox{ .}$$
\end{proof}

\begin{thm}
\label{teoreprecualtiempo}
Suppose that for any $1\le j \le t+1,$ the kernel $k^j$ satisfies the same hypotheses of Theorem~\ref{hipo}. Then, for any positive number $\epsilon$, there exist  positive integers $\overline{k_1}$, $\overline{k_2}$, and $\overline{k_3},$ such that for any natural numbers $k_1, k_2,k_3$, satisfying  $\overline{k_1}\le k_1,$ $\overline{k_2}\le k_2$ and $\overline{k_3}\le k_3$, the following inequality holds
$$|f_{k_1,k_2}(h^{t+1}_{k_3}(a),h^{t+1}_{k_3}(a),x,y)-(D^{t+1} (x,y))^2| \le \epsilon \mbox{ ,}$$
where $f_{k_1,k_2}$ is defined as in Lemma~\ref{aproximalem}, and $h^{t+1}_{A}$ as in Lemma~\ref{poqueal}.
\end{thm}

\begin{proof}
Applying Theorem~\ref{teoprin} to the kernel $k^{t+1}$, we have that there exist $\overline{k_1}$, $\overline{k_2}$, such that for any $\overline{k_1} \le k_1,$ and $\overline{k_2} \le k_2,$  the following inequality holds 
$$ |f_{k_1,k_2}(a^{t+1},a^{t+1},x,y )-(D^t (x,y))^2| \le \epsilon/2\mbox{ , } $$
for any $x,y \in X$. Moreover, using Lemmas  \ref{poqueal} and  \ref{continuitylemma}, there exists a positive integer $\overline{k_3}$ with the property that if $\overline{k_3}\le k_3$, then for all $x,y \in X,$
$$| f_{k_1,k_2}(a^{t+1},a^{t+1},x,y)- f_{k_1,k_2}(h^{t+1}_{k_3}(a),h^{t+1}_{k_3}(a),x,y) | \le \epsilon/2 \mbox{ .}$$
Thus, by the triangular inequality we obtain the desired result.
\end{proof}

\subsection{Diffusion representation for changing data}
In this section we use the coordinate system of Eq.~(\ref{tensorbasis}) to represent  the dynamic diffusion distance induced by asymmetric kernels. Suppose that $\{k_\gamma\}$ is a family of asymmetric kernels defined in the data set $X$. Again, we consider the diffusion distance without the Markov normalization, that is, we work with the dynamic diffusion distance given by
$$ D^t(x_\gamma,y_\beta)={\| \,  k_\gamma ^ t(x,\cdot)-k_\beta ^ t(y,\cdot)\|}_{L^2(X)}.$$
\par We assume that for each parameter $\gamma$, the kernel $k_\gamma$ satisfies all the hypotheses of Theorem~\ref{difurepteo}. In this case, we can write the kernel $k_\gamma$ as
$$  k_\gamma(x,y)= \sum_{(m_1,m_2)  \in  \mathbb{Z} \times \mathbb{Z}} a^{\gamma}_{(m_1 , m_2 )} \, W_{m_1} \otimes W_{m_2} (x,y).$$
\par The following theorem gives a representation for the dynamic diffusion distance. The proof of the theorem is similar to that of in Theorem~\ref{difurepteo}.
\begin{thm} \label{diftimech} Assume that for all $\gamma$, the kernel $k_\gamma$ satisfies all the hypotheses of Theorem~\ref{difurepteo}, then the dynamic diffusion distance $(D^t (x_\gamma,y_\beta))^2$ can be written as
$$\sum_{m_2 \in  \mathbb{Z}} \| \sum_{m_1 \in  \mathbb{Z}}  (a^{\gamma})^t(n,m) \,  W_{m_1}(x)-(a^{\beta})^t(n,m) W_{m_1}(y) \| ^2.$$
\end{thm}
\begin{rem} \label{rema}
Using the same ideas of the proof of Theorem~\ref{teoreprecualtiempo}, we can prove that it is possible to approximate the dynamic diffusion distance by  sums involving  few terms of $ a^{\gamma}$ and $a^{\beta}$ . To be more specific, the same statement of Theorem~\ref{teoreprecualtiempo} holds if we replace the classical diffusion distance by the dynamic diffusion distance $D^t (x_\gamma,y_\beta)$, and the function $f_{k_1,k_2}(h^{t+1}_C(a),h^{t+1}_C(a),x,y)$ by the function  $f_{k_1,k_2}(h^{t+1}_C (a^{\gamma}),h^{t+1}_C(a^{\beta}),x,y).$
\end{rem}

We use the previous result to compute the global diffusion distance in terms of the coefficients $a^{\gamma}$. We recall that the global diffusion distance between $X_\alpha,$ and $X_\beta,$ is defined by
$$ (\mathbb{D}^t(X_\gamma,X_\beta))^2=\int_{X} ({D}^t(x_\gamma,x_\beta))^2 d \mu (x).$$
\begin{thm}
Under the same assumptions of Theorem~\ref{diftimech}, the global diffusion distance at time $t$ can be written as
$$(\mathbb{D}^t(X_\gamma,X_\beta))^2=  \sum_{(n,m)  \in  \mathbb{Z} \times \mathbb{Z}}  \|((a^{\gamma})^t-(a^{\beta})^t)(n,m)\|^2 \mbox{ .}$$
\end{thm}

\begin{proof}
Theorem~\ref{diftimech} implies that
$$ (\mathbb{D}^t(X_\gamma,X_\beta))^2= \int_{X}  \sum_{m \in  \mathbb{Z}} \| \sum_{n \in  \mathbb{Z}}  ((a^{\gamma})^t-(a^{\beta})^t)(n,m) W_{n}(x) \| ^2 d \mu (x) \mbox{ .}$$
Expanding the quadratic form, we obtain that 
$$\mathbb{D}^t(X_\gamma,X_\beta)^2 = \int_{X} \sum_{m \in  \mathbb{Z}}   \sum_{(i,j) \in  \mathbb{Z} \times \mathbb{Z} }     ((a^{\gamma})^t-(a^{\beta})^t)(i,m) \, \overline{((a^{\gamma})^t-(a^{\beta})^t)(j,m)} \, W_{i}(x) \, \overline{W_{j}}(x)d \mu (x) \mbox{.}$$
Using Hölder's inequality and the fact that $\|W_n\|_{L^2(X)}=1$, we conclude that the expression 
$$  S=\sum_{m \in  \mathbb{Z}}   \sum_{(i,j) \in  \mathbb{Z} \times \mathbb{Z} }    \int_X \|((a^{\gamma})^t-(a^{\beta})^t)(i,m) \, \overline{((a^{\gamma})^t-(a^{\beta})^t)(j,m)} \, W_{i}(x) \, \overline{W_{j}}(x) \| d \mu (x). $$
is bounded from above by
$$  \sum_{m \in  \mathbb{Z}} \left(\sum_{n \in  \mathbb{Z}} \|((a^{\gamma})^t-(a^{\beta})^t)(n,m)\|  \right)^2 \le \left( \sum_{m \in  \mathbb{Z}}\sum_{n \in  \mathbb{Z}} \|((a^{\gamma})^t-(a^{\beta})^t)(n,m)\| )  \right)^2 < \infty \mbox{, }$$
where we used the fact that the $L^2$ norm is smaller than or equal to the $L^1$ norm. By the dominated convergence theorem, we can change the order between the integral and the sums. Moreover, using the fact that $W_n$ is an orthonormal basis of $L^2(X),$ we conclude that
$$(\mathbb{D}^t(X_\gamma,X_\beta))^2= \sum_{m \: \in  \: \mathbb{Z} }  \sum_{n\: \in  \: \mathbb{Z}} \|((a^{\gamma})^t-(a^{\beta})^t)(n,m)\|^2\mbox{ . } $$
\end{proof}

\subsection{Weak representation}
The framework developed up to this point uses the absolute summability condition (\ref{condilun}).  We now relax this assumption, and under an {\it a priori} assumption on the diffusion distance, we design a representation for the diffusion distance of points lying in a set of large measure. Without loss of generality, we work with the changing data framework.

\begin{thm} Let $t$ be a nonnegative integer, assume that the kernel $k^{t+1}_\gamma $ satisfies the Hypothesis  \ref{hipo}, and also that for any integer $j$, with  $1\le j \le t,$ the kernel $k^{j}_\gamma $ satisfies all the hypotheses of Theorem~\ref{difurepteo}. If we suppose that the dynamic diffusion distance at time $t+1$ is bounded from above, that is, there exists a positive constant $C$ such that for any $x,y \in X$, and all indices $\gamma, \beta ,$
$$D^{t+1} (x_\gamma,y_\beta) \le C.$$
Then, for any positive real number $\epsilon$, there exists a positive integer $k_0,$ such that for any integer $k\ge k_0$, there exists a measurable set $E_k$, with the property that the Lebesgue measure of $X \setminus E_{k}$ satisfies $\mathcal{L}(X \setminus E_{k})<\epsilon$, and such that for any $x,y \in E_{k}$ the following inequality holds
$$\|f_{k,k}(h^{t+1}_k (a^\gamma_ k),h^{t+1}_k(a^\beta_ k),x,y)-(D^{t+1} (x_\gamma,y_\beta))^2\| \le \overline{C} (\sqrt{\epsilon}+\epsilon)  \mbox{, }$$
where $\overline{C}$ is a constant and the functions $f_{k,k},$  $h^{t+1}_k$ are defined in Lemmas  \ref{aproximalem}, and  \ref{poqueal}, and $a^\gamma_ k$ is the truncated sequence defined by 
$$
a^\gamma_ k(n,m) = \left\{
     \begin{array}{@{}l@{\thinspace}l}
       \ a^\gamma(n,m)\mbox{, } &   \quad\text{if} \quad \|(n,m)\|\le k \\
       \ 0\mbox{, } &  \quad\text{otherwise} \mbox{ } \\

     \end{array}
   \right.
\mbox{. } $$  

\end{thm}

\begin{proof} Define the truncated kernel $k_{\gamma, k}$ by
 $$ k_{\gamma, k}(x,y)=\sum_{(m_1,m_2)  \in  \mathbb{Z} \times \mathbb{Z}} a^\gamma_ k{(m_1 , m_2 )} \, W_{m_1} \otimes W_{m_2} (x,y) \mbox{ . } $$ 
Using the recursive formula in Eq.~(\ref{recufor}), and the fact that $\lim_{k\to\infty}  k_{\gamma, k}= k_{\gamma}$, we see inductively that $\lim_{k\to\infty}  k^{t+1}_{\gamma, k}= k^{t+1}_{\gamma}$ for any non negative integer $t$, where the convergence is in the $L^2(X \times X)$ norm. Therefore, for any positive number $\epsilon$, there exists $k_0$ such that for any  integer $k$ satisfying $k_0 \le k,$ we have  $\|k^{t+1}_{\gamma}-k^{t+1}_{\gamma, k}\|_{L^2(X \times X)} \le \epsilon.$ Define the set 
$$E_k = \{ x\in X | \, \, \,\|k^{t+1}_{\gamma, k}(x, \cdot)-k^{t+1}_{\gamma}(x, \cdot),\|_{L^2(X)} \le  \sqrt{\epsilon} \} \mbox{ .}$$
Then, by Chebyshev's inequality and Fubini's theorem we obtain that $$\mathcal{L}(X \setminus E_{k})\le \frac{1}{\epsilon} \int_X \|k^{t+1}_{\gamma, k}(x, \cdot)-k^{t+1}_{\gamma}(x, \cdot),\|^2_{L^2(X)} dx \le \epsilon \mbox{ .}$$
On the other hand, if $x, y \in E_k $, then by Minkowski inequality
$$|D^{t+1} (x_{\gamma, k},y_{ \beta, k})-D^{t+1} (x_\gamma,y_\beta)| \le 2 \sqrt{\epsilon} \mbox{ , }$$
where $x_{\gamma, k}$ is the data point $x$ endowed with the kernel structure $k^{t+1}_{\gamma, k}$. We apply Inequality~ (\ref{desiguimport}) of Lemma   \ref{continuitylemma} with $B=D^{t+1} (x_{\gamma, k},y_{ \beta, k}),$ and $C=D^{t+1} (x_\gamma,y_\beta)$, to obtain
$$|D^{t+1} (x_{\gamma, k},y_{\beta, k})^2-D^{t+1} (x_\gamma,y_\beta)^2|\le (1+2M)(\epsilon+\sqrt{\epsilon}) \mbox{ .}$$
Since the set of non-zero coefficients of the kernel $k_{\gamma, k}$ is finite, then, by Eq.~(\ref{expansion}), we conclude that $a^{t+1}_{\gamma, k}(n,m)=0$ whenever $\|(n,m)\|>k,$ and also that for any integers $n,m$ we have $a^{t+1}_{\gamma, k}(n,m)= h^{t+1}_{k}(a)(n,m)$. Using these facts we conclude that
$$ D^{t+1} (x_{\gamma, k},y_{\beta, k})^2=f_{k,k}(h^{t+1}_{k} (a^\gamma_ k),h^{t+1}_{k} (a^\beta_ k),x,y) \mbox{ .}$$
\end{proof}

\section{Applications, Experiments, and Results}
\label{experimentos}
\subsection{Dimensionality reduction}
\label{capacidadpc}
In this section, we use our representation framework to reduce the dimension of data sets, in such way that the Euclidean norm of the reduced data set  approximates the diffusion distance.
Assume that the data set $X$ is endowed with the kernel structure $k_{\alpha}.$ Then, by Remark~\ref{rema} the map $\phi_{k_1,k_2}: X \to \mathbb{C}^{2k_2 +1}$ defined by
\begin{equation}
    \label{redudimen}
     \phi_{k_1,k_2}(x)(i)=\sum_{\|m_1\|\le k_1 }  (a^{\gamma})^t(m_1,i-k_2-1) \,  W_{m_1}(x)  ,
\end{equation}
approximates the diffusion distance at time $t$. We summarize the above in Algorithm~\ref{algoritmo}.
\par In order to work with a small parameter $k_2$, and thus embed our data set in a low dimensional Euclidean space, we need to use  a proper orthonormal basis of  $L^2(X)$. Proper bases are those for which the information of the kernel $k$ is concentrated on the coefficients $a_{(m_1,m_2)},$ whose pair of integers $(m_1,m_2)$ are near to the origin.

{ \centering
\begin{minipage}{ 1 \linewidth}
  \begin{algorithm}[H] 
   \caption{Dimensionality reduction algorithm. \label{algoritmo}}
\begin{enumerate}
      \item Take the data in the form of an $M \times M$ (possibly asymmetric) matrix $L$ representing the kernel structure on the data set  
      \item Run the Markov process $t$-times, and use the matrix $L^t$ instead of $L$.
      \item Compute the first components of the matrix $L^t$ in a proper orthonormal tensor basis.
      \item Embed the data set using the function $\phi_{k_1,k_2}$ of Eq.~(\ref{redudimen})
    \end{enumerate}
  \end{algorithm}
\end{minipage}
\hfill \break
\par
}
\par 
We note that depending on the data and the mathematical model, we may want to use a proper normalization for the matrix $L$ in Algorithm~\ref{algoritmo}.

In practical situations, our measure space is a finite data set $X=(x_i)^{n}_{i=1}$, endowed with the counting measure. If the kernel matrix $\textbf{A}$ is symmetric, then one can use the eigenvector basis. More generally, if the matrix $\textbf{A}$ is asymmetric, one can use the singular value decomposition (SVD) to write 
$$\textbf{A}=\sum_{i=1}^{n} c_i  \, L_i \otimes R_i\mbox{ , } $$
where $ L_i$ and $R_i$ are the left singular vectors and right singular vectors of $\textbf{A}$, respectively, and where $c_i$ are the singular values of $\textbf{A}$.  The computational complexity of the SVD of an $n\times n$ matrix is $O(n^3)$. In the following experiments, we use the Fourier basis to improve the computational complexity. In fact, the computational complexity of the 2-D FFT is $O(n^2 \log n).$  We recall that the Fourier basis $(E_k)^{n-1}_{k=0}$ in the complex vector space $\mathbb{C}^{n}$ is given by
\begin{equation}
    \label{fourier}
    E_k=\frac{1}{\sqrt{n}}(1, e^{2 \pi i k / n}, e^{2 \pi i 2 k / n}, ... , e^{2 \pi i(n-1) k / n})
\mbox{ .} \end{equation}
\par In all experiments we compute the coefficients using the 2D-FFT algorithm.  All experiments were run in \MATLAB software, using a desktop computer with the following configuration: Intel core i7-2600 3.4 GHz processor, and 16 GB RAM.
\subsection{Synthetic data using a symmetric kernel}

In this experiment, our data set $X$ consists of $n$ random points in the sphere $S^2$ (Figure~\ref{esfera}). Here, we use the parametrization  
\begin{alignat*}{4} 
x(u,v) & {}={} &  \cos{u} & \sin{v} \mbox{, }  \\
y(u,v) & {}={} &  \sin{u} & \sin{v} \mbox{, }  \\
z(u,v) & {}={} &  \cos{v} \mbox{, } 
\end{alignat*}
for $0 \le u \le  \pi$ and $0 \le v \le 2 \pi$. We endow the data set $X$ with the Markov normalization $\rho$ defined in Eq.~(\ref{kerma}), of the Gaussian kernel
$$ k(x,y)=e^{- ||x-y||^2} \mbox{ .}$$
\begin{figure}
 \centering
\includegraphics[scale=0.7]{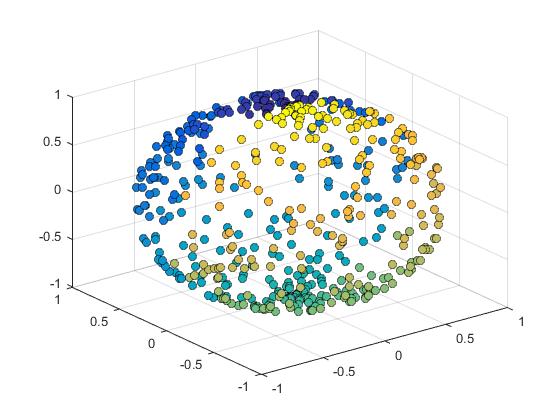}
\caption{Data set $X$ with $512$ random points in the sphere $S^2$.}
\label{esfera}
\end{figure}
Our goal is to compare the efficiency between the representation given in Eq.~(\ref{spede}) using the eigenvector basis and the representation of Theorem \ref{difurepteo}, using the Fourier basis of Eq.~(\ref{fourier}). In
Figure~\ref{experisintetico}, we show the first two coordinates for each representation with a data set of 512 points. Observe that the first two coordinates of both representations are similar.
\begin{figure}[htp] 
    \centering
    \subfloat[ Eigenvector basis]{%
        \includegraphics[width=0.33\textwidth]{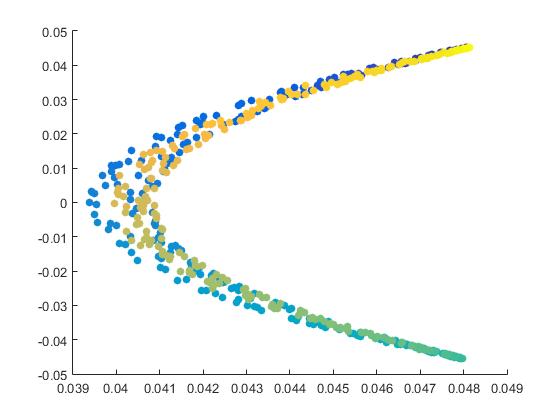}%
        }%
    \subfloat[Real and imaginary part of the Fourier basis]{%
        \includegraphics[width=0.33\textwidth]{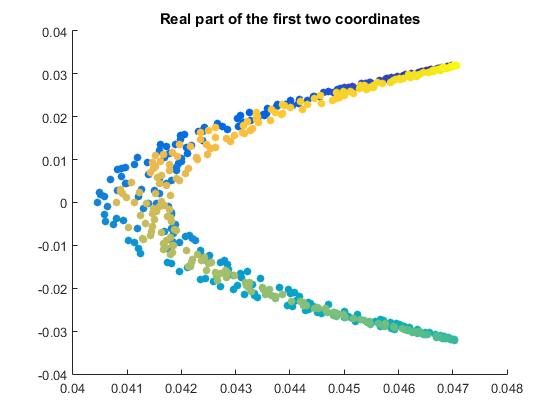}%

        \includegraphics[width=0.33\textwidth]{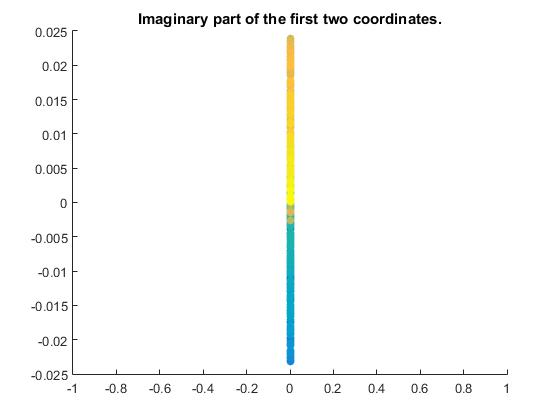}%
       }%

    \caption{Plot of the two dimensional embedding for the data set $X$  using the eigenvector basis coefficients (a), and the Fourier basis coefficients (b). Note the scale.}
    \label{experisintetico}
\end{figure}
In Figure~\ref{errorsintetico}, we plot the error, and computational time (in seconds) of the first two
coordinates for several values of $n$. We also remark that by using the Fourier basis, the performance of the representation is faster than using the eigenvector basis, and also that the Fourier basis gives an acceptable error when compared to the eigenvector method.
\begin{figure}[htp] 
    \centering
    \subfloat[Error of the approximation]{%
        \includegraphics[width=0.5\textwidth]{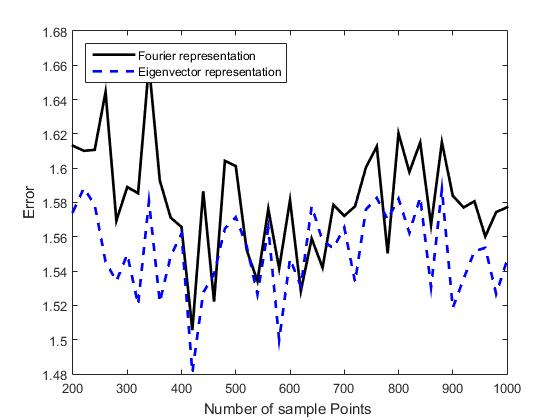}%
        }%
    \subfloat[Computational time]{%
        \includegraphics[width=0.5\textwidth]{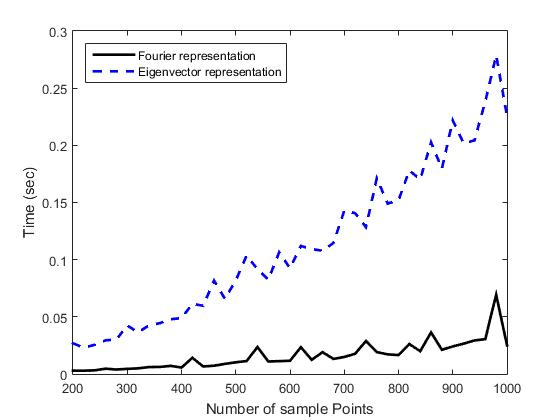}%
        }%
       
    \caption{Plot of the $L^2$ error and  computational time of the first two coordinates for different $n\times n$ kernel-sizes, for the data set of random points in the sphere.}
    \label{errorsintetico}
\end{figure}

\subsection{Synthetic data on the Möbius strip}

Here, we assume that our data set $X$ is a set of $300$ data points distributed along of the Möbius strip (Figure~\ref{mobius}).  We use the parametrization  
\begin{alignat*}{4} 
x(u,v) & {}={}   (1+\frac{v}{2} \cos{\frac{u}{2}} ) \cos{u} \mbox{, } \\
y(u,v) & {}={}   (1+\frac{v}{2} \cos{\frac{u}{2}} ) \sin{u} \mbox{, } \\
z(u,v) & {}={}   \frac{v}{2} \sin{\frac{u}{2}} \mbox{, }
\end{alignat*}
for $0 \le u \le 2 \pi$ and $-\frac{1}{2} \le v \le \frac{1}{2} $. We endow the data set $X$ with the Markov normalization $\rho$ defined in Eq.~(\ref{kerma}),
of the weight Gaussian kernel
\begin{figure}[h!] 
 \centering
\includegraphics[scale=0.7]{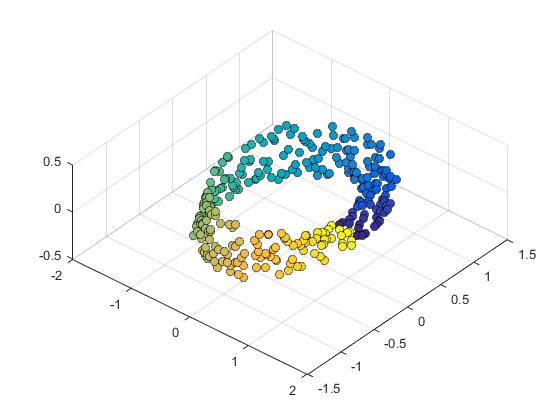}
\caption{Data set $X$ with $300$ random points in the M\"obius strip $M^2$.}
\label{mobius}
\end{figure}
$$ k(x,y)=(S(x-y)+1) e^{- ||x-y||^2}\mbox{, }$$
where $S(z)$ is the sign function  of the angle (in cylindrical coordinates) of the vector $z$. This kernel measures local information taking into account if the first two components are rotating clockwise.
In this experiment, we compared the performance of the representation using the SVD, and the representation using the Fourier basis. In Figure~\ref{mergu}, we plot the first two coordinates for each representation. Note that the real part of the representation given by the Fourier basis allows us to see in more detail the distribution of the data set. In fact,  the representation that uses the Fourier basis recognizes the rotation of the data set. However, the presentation that uses SVD does not allow recognizing this feature of $X$. This is due to the fact that the representation using the SVD approximates the kernel $ k ^ T k $, instead of the kernel $ k $ (Figure~\ref{aproximkersime}). Therefore, the representation using the SVD does not distinguish some geometric properties of the data set $ X $.
\begin{figure}[htp] 
    \centering
    \subfloat[ svd]{%
        \includegraphics[width=0.33\textwidth]{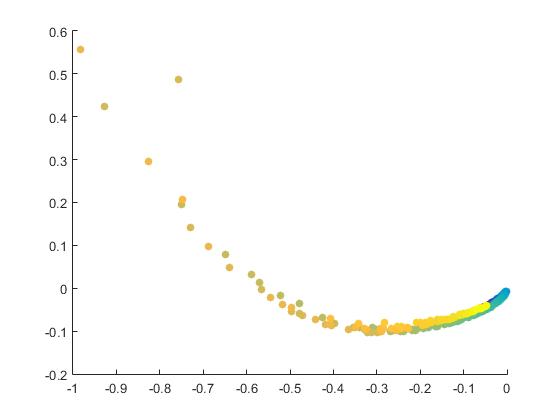}%
        }%
    \subfloat[Real and imaginary part of the Fourier basis]{%
        \includegraphics[width=0.33\textwidth]{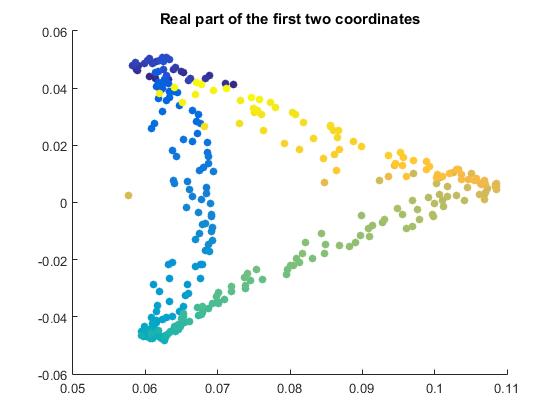}%

        \includegraphics[width=0.33\textwidth]{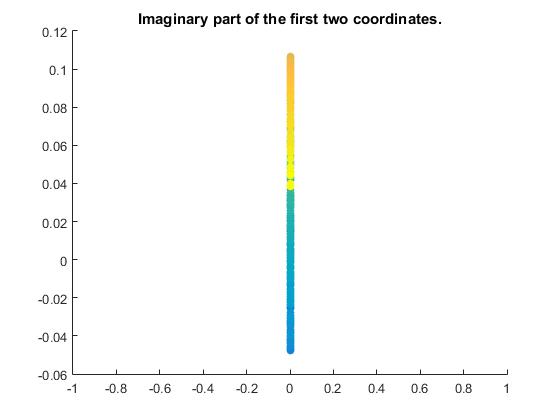}%
       }%
       \caption{Dimensionality reduction using the Eigenvector basis, and Fourier basis}
       \label{mergu}
       \end{figure}

       \begin{figure}
    \centering
    \subfloat[ Kernel matrix $k$]{%
        \includegraphics[width=0.5\textwidth]{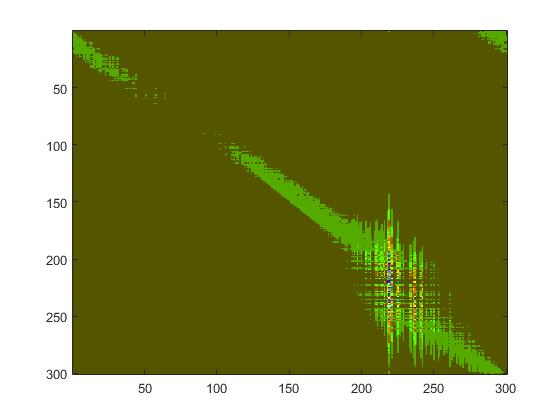}%
        }%
   
   \subfloat[ Kernel matrix $k^T k$]{%
        \includegraphics[width=0.5\textwidth]{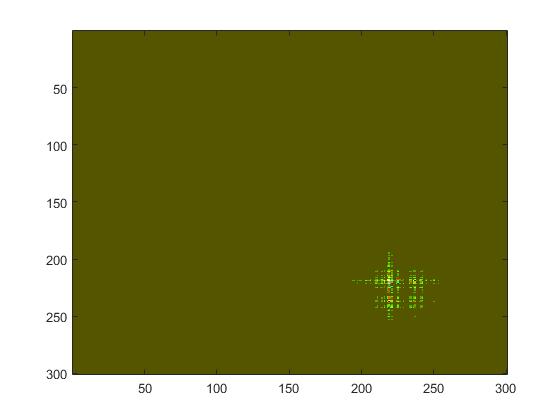}%
        }%

          \caption{Kernel matrix $k$ (a), and normalization $k^T k$ (b)}
           \label{aproximkersime}
       \end{figure}

\subsection{Synthetic data using an asymmetric
kernel}

Here, we assume that our data set $X$ is a random set of $256$ data points. We endow this data set with the kernel structure $k$  given by the Tom Jobim picture of Figure~\ref{tomjobim}  whose dimensions are $256 \times 256$ pixels. That is, $k(x,y)$ is defined to be the gray scale value of the pixel coordinates $(x,y)$. As in the previous experiment, we use the Markov normalization of the kernel $k$. In this experiment, we compared the performance of the representation using the SVD and the representation using the Fourier basis.
\par We stress that our objective with this example is not to try to do image processing, but rather to use a picture so that we are able to visually assess the quality of the approximation. This point will be furthered in the sequel.  
\par In Figure~\ref{experisinteticoapro}, we plot the approximation of the kernel using the SVD and the Fourier basis, both using the parameters $k_1=256$ and $k_2=64$. One may notice that we see a horizontal modulation both under the SVD and the Fourier basis methods. However, this modulation is stronger in the Fourier method. This is due to the fact that we have only used high frequencies to approximate this image. 
\par In this case we observe that despite using a smaller number of  parameters, it is possible to obtain a good approximation of the original image. In  Figure  \ref{errortom}, we plot the $L^2$ error, and computational time (in seconds), in a logarithmic scale, of the embedding data set using several approximation orders. As in the previous experiment, we see that using the Fourier basis, the performance is faster, and provides an acceptable error when compared to the SVD method. Furthermore, we point out that in the two previous experiments we did not obtain a better performance with respect to computational time when we used the truncated SVD instead of the SVD.

\begin{figure}
 \centering
\includegraphics[scale=0.7]{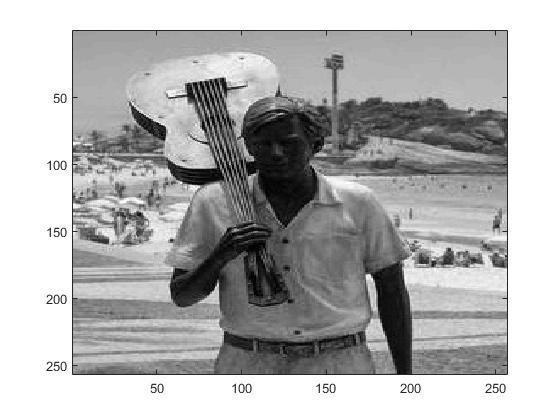}
\caption{Synthetic  asymmetric kernel structure chosen from a picture of Tom Jobim in Ipanema Beach (Rio de Janeiro, Brazil). Source \cite{Tomjo}.}
\label{tomjobim}
\end{figure}

\begin{figure}[htp] 
    \centering
    \subfloat[Approximation using the SVD]{%
        \includegraphics[width=0.7\textwidth]{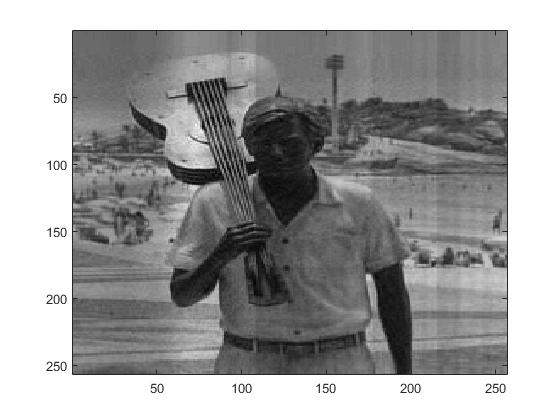}%
        }%
    \hfill%
    \subfloat[Approximation using the Fourier basis]{%
        \includegraphics[width=0.7\textwidth]{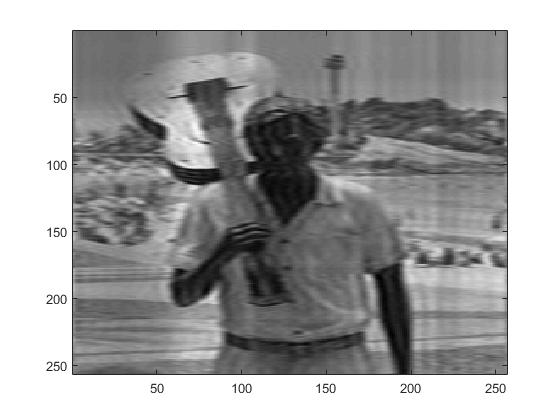}%
       
        }%

    \caption{Plot of the approximation using the SVD and the Fourier basis, both using the parameters $k_1=256$ and $k_2=64$}
    \label{experisinteticoapro}
\end{figure}

\begin{figure}[htp] 
    \centering
    \subfloat[Approximation error]{%
        \includegraphics[width=0.7\textwidth]{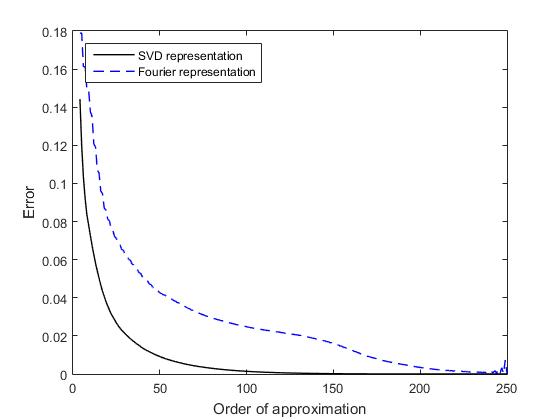}%
        }%
    \hfill%
    \subfloat[Computational time]{%
        \includegraphics[width=0.7\textwidth]{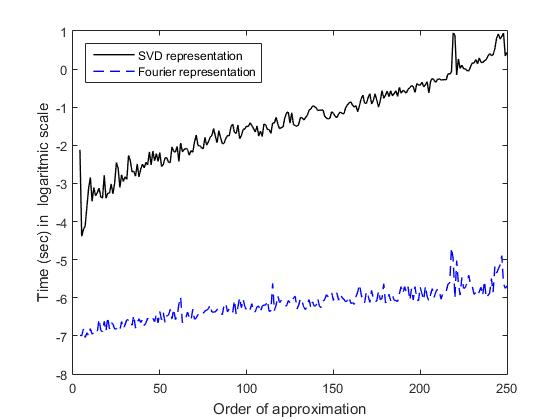}%
       
        }%

    \caption{Plot of the $L^2$ error and computational time for the data set representing the picture of Tom Jobim.}
    \label{errortom}
\end{figure}

\subsection{Temperature changes in Brazil}
In the last decades the world temperature distribution has presented drastic changes,  in part this is likely due to human activities \cite{climacamb, climacamb1,acuepar}. Here, we use the diffusion distance for changing data to detect the regions of Brazil in which the local temperature variation was the highest in the years 2010 and 2018 (Figure~\ref{anosmapas}),  compared with the year 2000 (Figure~\ref{anosmapas}). In fact, if a certain region has a great diffusion distance, it means that this particular region has presented significant changes in its temperature. This experiment is based on the change detection on hyperspectral imagery data proposed in Ref.~\cite{COIFMAN201479}. Our data set consists of $N=13,974$ points, and  each point represents a pixel coordinate of the Brazilian map. These points are a subset of a picture of size $170 \times 170$ pixels.
Here we do not take into account the blank pixels, which correspond to places outside the Brazilian soil. For each year we endow the data set with the un-normalized  kernel $K$ which is defined on $ X \, \times \, X$ by
$$K(x_{i,j},x_{k,l})=\frac{1}{\sqrt{N}} T_{\alpha}(x_{k,l}) \, e ^ {-\|(i,j)-(k,l)\|^2 / 2 \sigma^2}\mbox{ , } $$ 
where $T_{\alpha}(x_{k,l})$ is the temperature in the rectangular pixel $x_{k,l}$ in the year $\alpha  \in$  \{2000, 2010, 2018\}, and where $\|\cdot\|$ is the Euclidean distance. In this experiment we use the scaling parameter $2 \sigma^2=650$. We obtained similar results with a parameter in a range of $600 \le 2 \sigma^2 \le 700$.
We use this kernel without normalization in order to avoid its high computational cost. This data set was taken from the Brazilian National Institute of Meteorology website \cite{climain}. 
This asymmetric kernel represents the distribution of the local temperature around the rectangular pixel $x_{k,l}$.
\par We use Theorem~\ref{diftimech} to approximate the dynamic diffusion distance $(D^1 (x_{2000},x_{\gamma}))^2$, for $\gamma \in \{2010, 2018 \}$. Due to the high dimensionality of the kernel matrix, the SVD algorithm did not run in the computer whose configuration is given in Section  \ref{capacidadpc}. Therefore, we cannot use the singular vector basis to represent the diffusion distance. Here, we use the Fourier basis defined in Eq. $(\ref{fourier})$ to represent this diffusion distance. We approximate the dynamic diffusion distance using the parameters $k_1=13,974$, and  $k_2 \in \{5,100\}$. See Figures  \ref{experi2010} and   \ref{experi2018} for $2010$ and $2018$, respectively. 
The green-yellow scale represents the intensity of the dynamic diffusion distance, in which the yellow regions have a greater diffusion distance as compared to the green regions. To detect which regions have the greatest positive variation, that is, regions in which the temperature has increased, we use contour plots of the diffusion distance, taking into account regions where the temperature increased. See Figures  \ref{experi2010map} and  \ref{experi2018map} for 2010 and 2018, respectively. 
\par In Table  \ref{tabla}, we show the global diffusion distance for each year. Observe that the distance is greater in 2018 than in 2010. This suggests that during 2018, there were more changes in temperature when compared to  2010.  
In Figures~\ref{error2010} and  \ref{error2018}, we plot the error and computational time of the performance for several approximation orders $k_2$, for 2010 and 2018, respectively. We evaluate the performance of the orders using the metric

\begin{figure}[htp] 
    \centering
    \subfloat[Year $2000$]{%
        \includegraphics[width=0.4\textwidth]{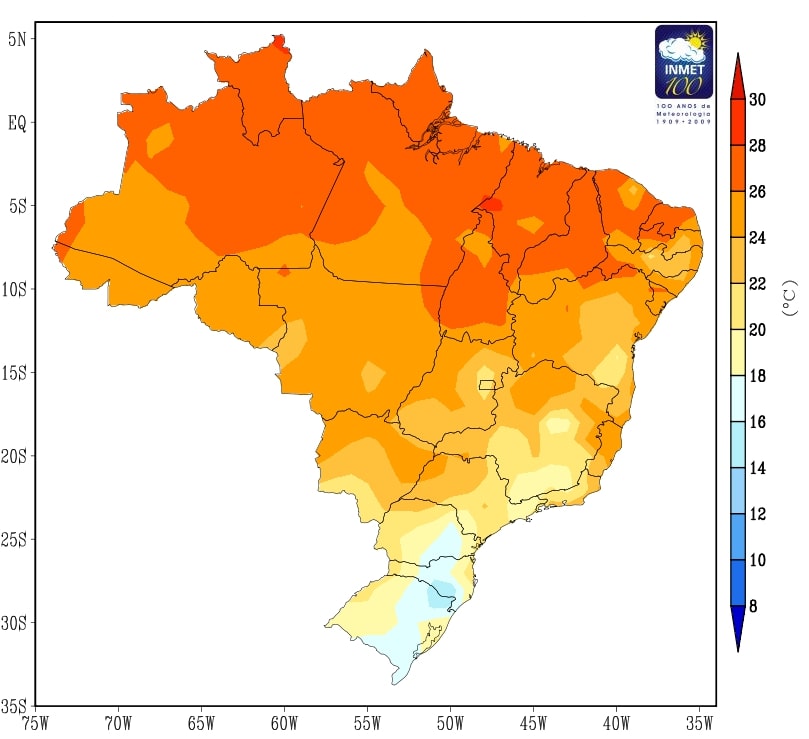}%
     }%
    \subfloat[Year $2010$]{%
        \includegraphics[width=0.4\textwidth]{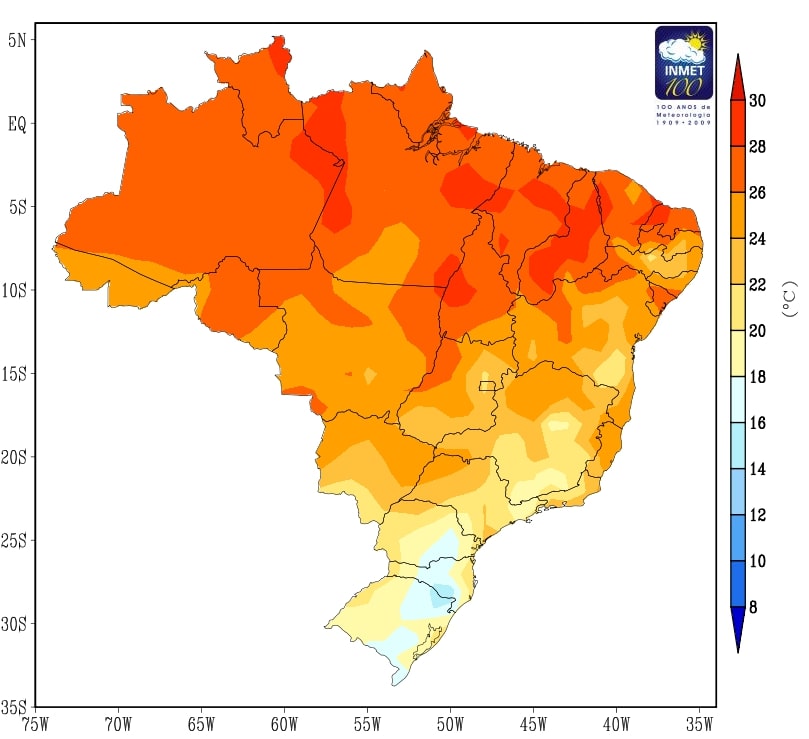}%
  
        }%
        \hfill
        \subfloat[Year $2018$]{%
        \includegraphics[width=0.4\textwidth]{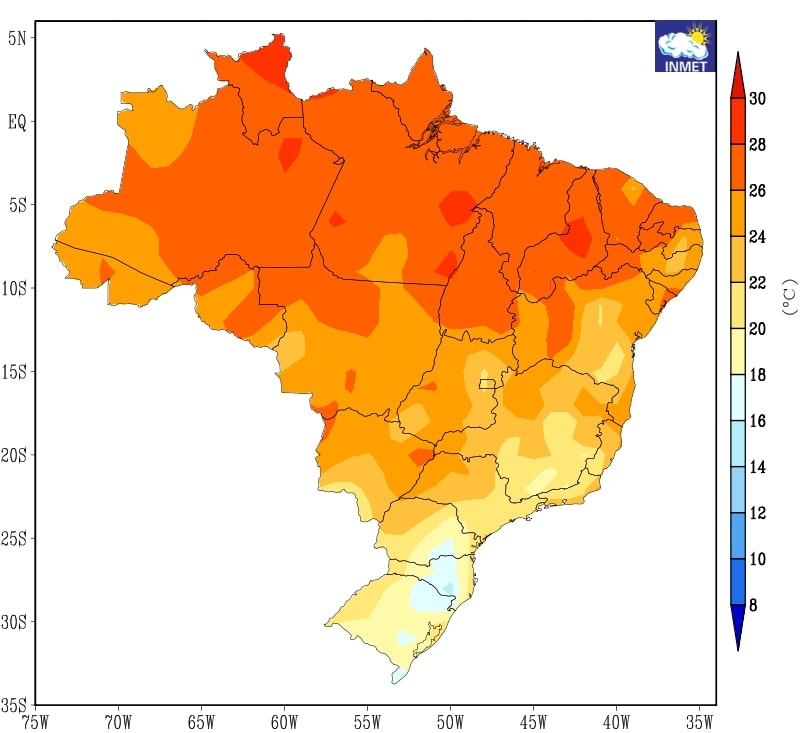}%
    
        }%
    \caption{Plot of the average Brazilian temperature distribution in years $2000$, $2010$, $2018$.}
    \label{anosmapas}
\end{figure}

\begin{equation}
    \label{metrica}
   M_{B}= \frac{1}{N} E \times t_{cpu}
\mbox{ ,} \end{equation}
where $E$  the absolute error between $(D^1 (x_{2000},x_{\gamma}))^2$ and its approximation, and $t_{cpu}$ is the computational time (in seconds) to compute the approximation. We see that even using  small orders, it is possible to obtain a good performance compared to larger  orders.

\begin{table}
\centering
 \begin{tabular}{| c| c|} 
\hline
Year &  Global diffusion distance $(\;^{\circ} \mathrm{C})$ \\
\hline
2010   & 25.5506  \\
\hline
 2018  & 37.5000  \\
\hline
\end{tabular}
\caption{Global diffusion distance in the years 2010 and 2018, with respect to year 2000.}
\label{tabla}
\end{table}

\begin{figure}[!htb]
\minipage{0.32\textwidth}
  \subfloat[$k_2=5$]{%
  \includegraphics[width=\linewidth]{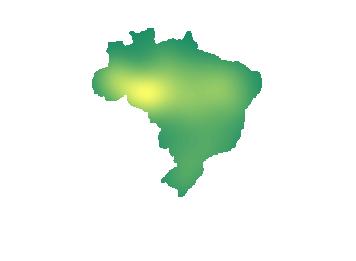}
}%
\endminipage\hfill
\minipage{0.32\textwidth}
    \subfloat[$k_2=100$]{%
  \includegraphics[width=\linewidth]{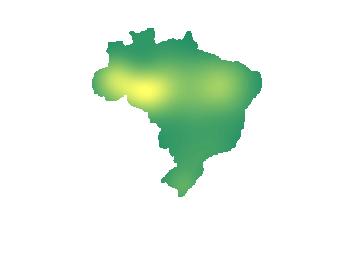}
}%
\endminipage\hfill
\minipage{0.32\textwidth}%
  \subfloat[ $D^1 (x_{2000},x_{2010})^2$]{%
  \includegraphics[width=\linewidth]{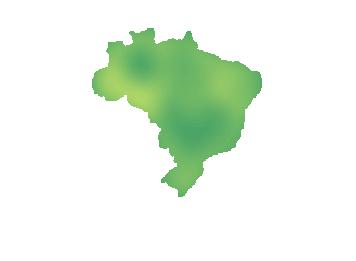}
}%

\endminipage

\caption{Plot of the approximations for the dynamic diffusion distance $(D^1 (x_{2000},x_{2010}))^2$ with different values of $k_2$ (a), (b), and plot of the dynamic diffusion distance (c).}
    \label{experi2010}
    \end{figure}
\begin{figure}[!htb]

\minipage{0.32\textwidth}
  \subfloat[$k_2=5$]{%
  \includegraphics[width=\linewidth]{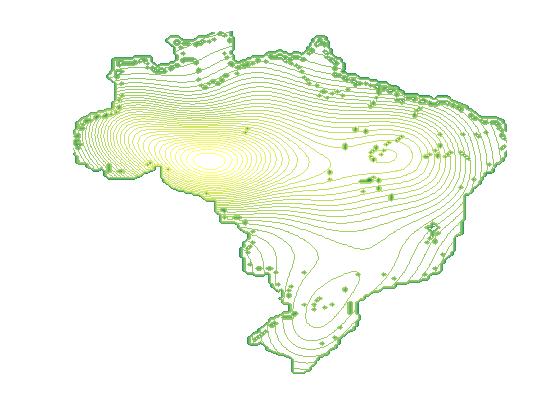}
}%
\endminipage\hfill
\minipage{0.32\textwidth}
    \subfloat[$k_2=100$]{%
  \includegraphics[width=\linewidth]{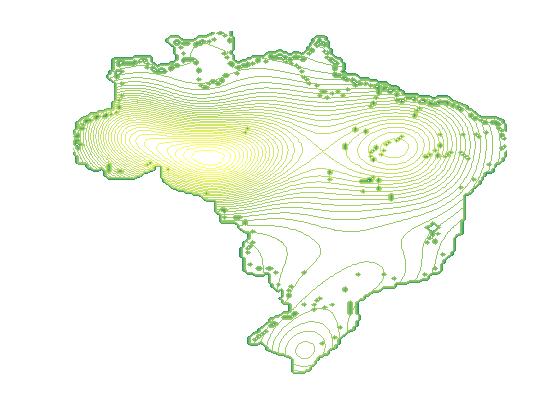}
}%
\endminipage\hfill
\minipage{0.32\textwidth}%
  \subfloat[ $D^1 (x_{2000},x_{2010})^2$]{%
  \includegraphics[width=\linewidth]{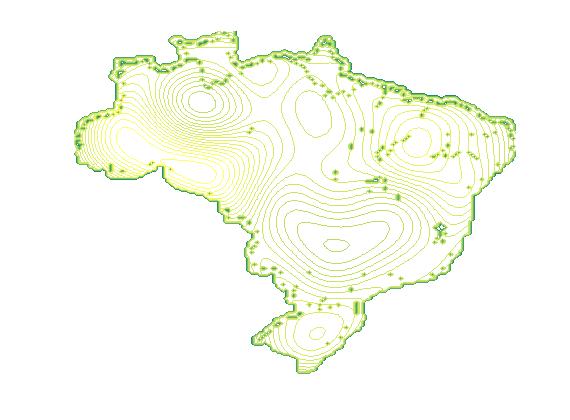}
}%

\endminipage

    \caption{Contour plot of the representation approximation for the diffusion distance $(D^1 (x_{2000},x_{2010}))^2$ taking into account the temperature increase.}
    \label{experi2010map}
\end{figure}

\begin{figure}[!htb]
\minipage{0.32\textwidth}
  \subfloat[$k_2=5$]{%
  \includegraphics[width=\linewidth]{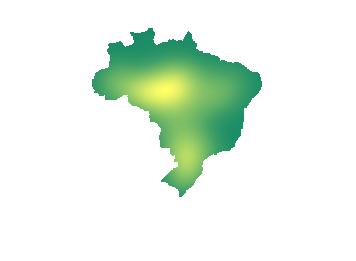}
}%
\endminipage\hfill
\minipage{0.32\textwidth}
    \subfloat[$k_2=100$]{%
  \includegraphics[width=\linewidth]{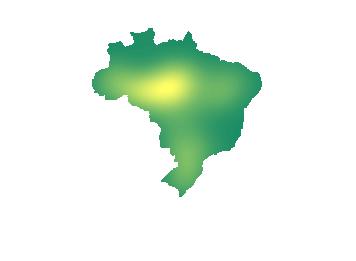}
}%
\endminipage\hfill
\minipage{0.32\textwidth}%
  \subfloat[ $D^1 (x_{2000},x_{2010})^2$]{%
  \includegraphics[width=\linewidth]{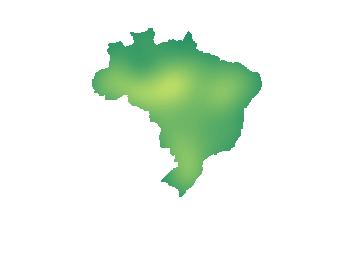}
}%

\endminipage

\caption{Plot of the approximations for the dynamic diffusion distance $(D^1 (x_{2000},x_{2018}))^2$ with different values of $k_2$ (a), (b), and plot of the dynamic diffusion distance (c).}
    \label{experi2018}
    \end{figure}
\begin{figure}[!htb]

\minipage{0.32\textwidth}
  \subfloat[$k_2=5$]{%
  \includegraphics[width=\linewidth]{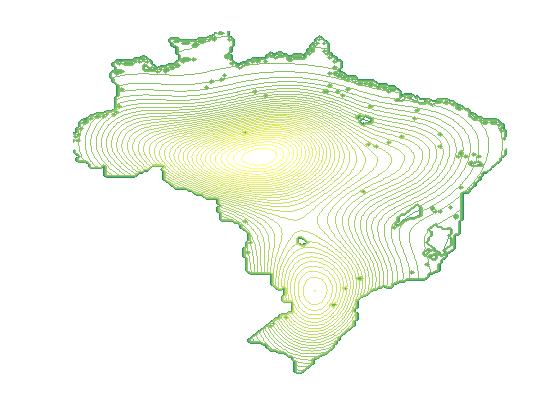}
}%
\endminipage\hfill
\minipage{0.32\textwidth}
    \subfloat[$k_2=100$]{%
  \includegraphics[width=\linewidth]{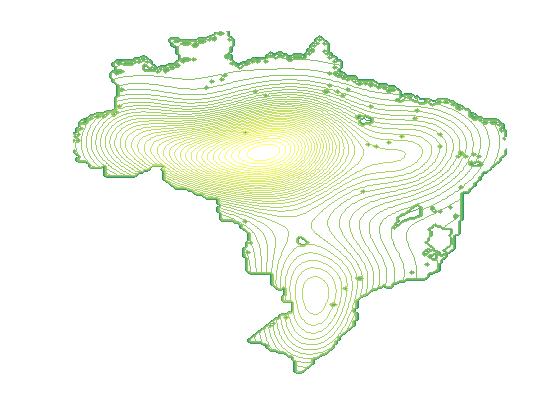}
}%
\endminipage\hfill
\minipage{0.32\textwidth}%
  \subfloat[ $D^1 (x_{2000},x_{2010})^2$]{%
  \includegraphics[width=\linewidth]{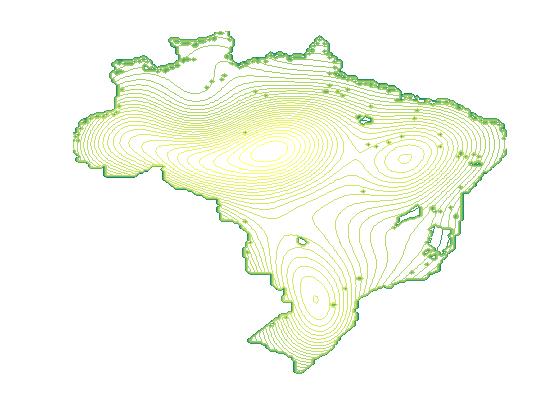}
}%

\endminipage

    \caption{Contour plot of the representation approximation for the diffusion distance $(D^1 (x_{2000},x_{2018}))^2$ taking into account the temperature increase.}
    \label{experi2018map}
\end{figure}


\begin{figure}[h!] 
    \centering
    \subfloat[Error of the approximation]{%
        \includegraphics[width=0.5\textwidth]{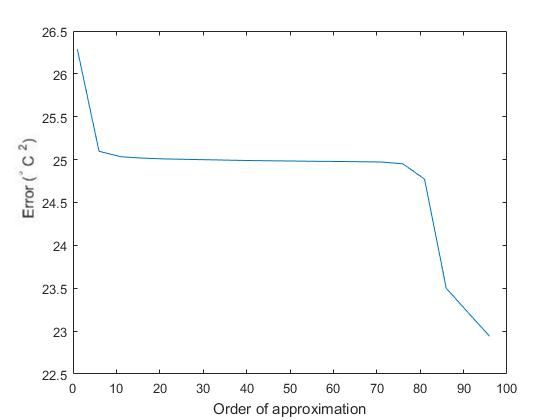}%
     }%
    \subfloat[Computational time]{%
        \includegraphics[width=0.5\textwidth]{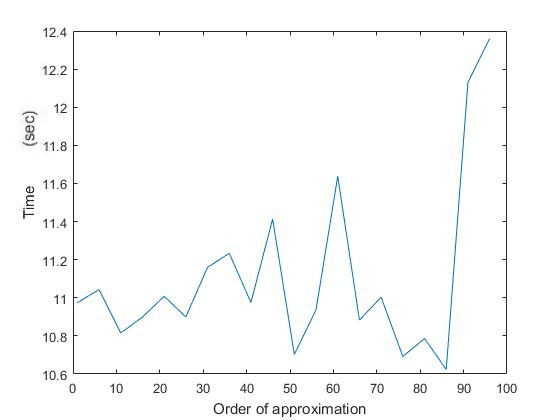}%
  
        }%
        \hfill
        \subfloat[Metric $M_{B}$]{%
        \includegraphics[width=0.5\textwidth]{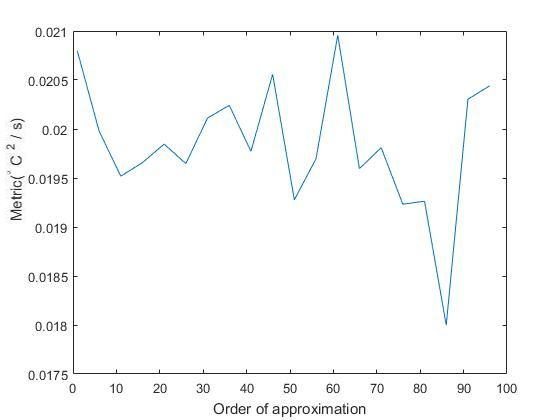}%
    
        }%
    \caption{Plot of the absolute error, computational time, and the metric given by Eq.~(\ref{metrica}) with different orders to approximate $(D^1 (x_{2000},x_{2010}))^2$, for the data set with the temperature distribution in Brazil.
     \label{error2010}
    }
   
\end{figure}

\begin{figure}[h!] 
    \centering
    \subfloat[Error of the approximation]{%
        \includegraphics[width=0.5\textwidth]{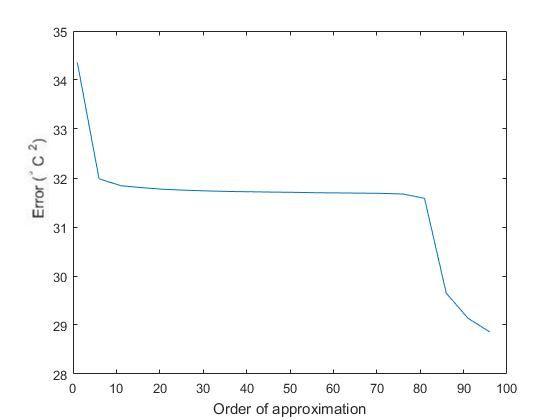}%
     }%
    \subfloat[Computational time]{%
        \includegraphics[width=0.5\textwidth]{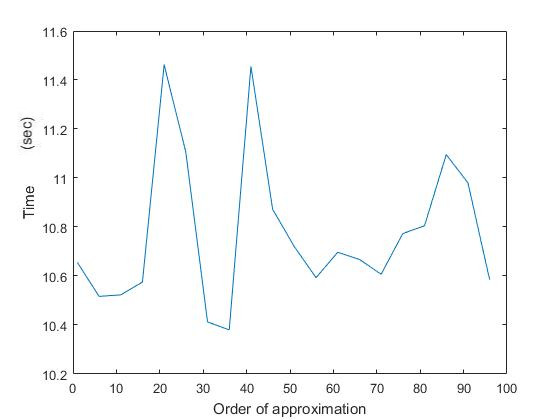}%
  
        }%
        \hfill
        \subfloat[Metric $M_{B}$]{%
        \includegraphics[width=0.5\textwidth]{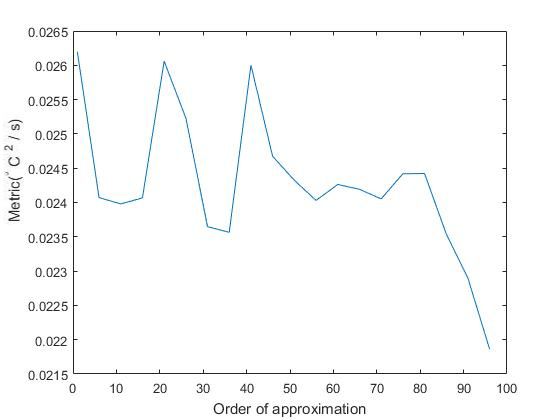}%
    
        }%
    \caption{Plot of the absolute error, computational time, and the metric given by Eq.~(\ref{metrica}) with different orders to approximate $(D^1 (x_{2000},x_{2018}))^2$,  for the data set with the temperature distribution in Brazil.\label{error2018}}
\end{figure}

\section{Conclusions}
\label{conclu}

In this paper, we treat the problem of dimensionality reduction of data sets whose structure is given by asymmetric kernels.
Our methodology generalizes the diffusion-map framework to asymmetric kernels, and computes a diffusion representation based on the kernel coordinates in a proper orthonormal basis. Our representation depends on two parameters, the first parameter defines the approximation error and the second one the dimensionality.
\par In our experiments, we  used the Fourier basis to represent the structure of the data set. This choice is based on the fact that the Fourier basis diagonalizes the Laplacian operator which is the main example of a diffusive process.
From the numerical viewpoint, the main advantage of using the Fourier basis is that the 2d-FFT allows us a reduction from linear growth to logarithmic growth of one of the factors. The latter contributes to the computational complexity reduction 
when compared to traditional eigenvalue methods. In fact, if we consider that our kernel is represented by an $n\times n$ matrix, then the SVD takes  $O(n^3)$ of operations to be performed, whereas the 2d-FFT decomposition is $O(n^2 \log(n))$. This fact was confirmed  in a set of experiments with randomly generated kernels. 
Observe that the SVD representation gives a better approximation, i.e., smaller errors. However, if we use the Fourier basis we can obtain a good approximation of the data set for a much lower computational cost. Additionally, the use of the Fourier basis allows to see in more detail some geometric properties of the data set.
This suggests that it is possible to use the Fourier basis as 
an alternative to the classic representation by eigenvalues, especially in computers with low performance.
\par We perform a few applied experiments to test the theory.
In particular, we  apply it to identify which regions of Brazil have presented a greater variation in the temperature {\em vis a vis} other ones during the last decades. In this experiment, we see that the Amazon region has presented  more variations in its temperature as compared to other places. This observation indicates that further studies should be performed to investigate the possible reasons for such variations.
In order to avoid increasing the computational cost, we did not use the  Markov normalization in this kernel. Due to the high dimensionality of the kernel matrix, the SVD algorithm did not run in our computer for this experiment. However, we managed to execute the algorithm using the 2d-FFT. 
\par We performed also some experiments with synthetic data using a wavelet basis. However, we did not obtain an improvement in the computation time or in the error of the approximation when compared to the Fourier basis and the singular vector basis. 

 \par Asymmetric kernels are present in a number of mathematical models, for instance in weighted directed graphs. In such graphs, the transition from one node to another is measured by an asymmetric kernel. 
 In general, asymmetric kernels are useful to represent a gain or loss of information when we move from one point to another.
 Weighted directed graphs are thus used to model real world problems such as, the traffic in a city, electrical network systems, water flow in hydrological basins, and commodity trading between economies. 
 These are natural follow-up avenues to the present work.
Another natural follow-up  would be the use of other
orthonormal bases. One possibility would be the diagonalizers of the Laplace-Beltrami operator on certain manifolds. An example of such bases is given by spherical harmonics.

\section*{Acknowledgements}
The authors acknowledge the financial support provided by CAPES,  Coordenação de Aperfeiçoamento de Pessoal de Nível Superior (Finance code 001), CNPq, Conselho Nacional de Desenvolvimento Científico e Tecnológico, and FAPERJ, Fundação Carlos Chagas Filho de Amparo à Pesquisa do Estado do Rio de Janeiro.

\section*{References} 
\bibliography{bibli}
\end{document}